\documentclass[runningheads]{llncs}

\usepackage[T1]{fontenc}
\usepackage{graphicx}
\usepackage{amsmath,amssymb}
\usepackage{listings}
\usepackage{xcolor}
\usepackage{tikz}
\usepackage{pgfplots}
\pgfplotsset{compat=1.18}
\usetikzlibrary{shapes.geometric,arrows.meta,positioning,fit,backgrounds,calc}
\usepackage{booktabs}
\usepackage{multirow}
\usepackage{microtype}
\usepackage{url}
\usepackage{mdframed}

\definecolor{codebg}{RGB}{246,248,250}
\definecolor{codegreen}{RGB}{0,128,0}
\definecolor{codegray}{RGB}{128,128,128}
\definecolor{codepurple}{RGB}{128,0,128}
\definecolor{codeblue}{RGB}{0,0,180}
\definecolor{histbg}{RGB}{255,250,240}

\lstdefinestyle{pythonstyle}{
  backgroundcolor=\color{codebg},
  commentstyle=\color{codegreen}\itshape,
  keywordstyle=\color{codeblue}\bfseries,
  stringstyle=\color{codepurple},
  basicstyle=\ttfamily\scriptsize,
  breaklines=true,
  captionpos=b,
  keepspaces=true,
  numbers=left,
  numbersep=5pt,
  numberstyle=\tiny\color{codegray},
  frame=single,
  rulecolor=\color{codegray!30},
  xleftmargin=14pt,
  framexleftmargin=14pt,
  tabsize=2,
  language=Python
}
\lstdefinestyle{sqlstyle}{
  backgroundcolor=\color{codebg},
  keywordstyle=\color{codeblue}\bfseries,
  basicstyle=\ttfamily\scriptsize,
  breaklines=true,
  frame=single,
  rulecolor=\color{codegray!30},
  morekeywords={CREATE,TABLE,IF,NOT,EXISTS,INTEGER,TEXT,PRIMARY,KEY,
                AUTOINCREMENT,FOREIGN,REFERENCES,CHECK,DEFAULT,ON,CONFLICT,
                DO,UPDATE,SET,NULL,SELECT,FROM,WHERE,ORDER,BY,INSERT,INTO},
  numbers=left,
  numbersep=5pt,
  numberstyle=\tiny\color{codegray},
  xleftmargin=14pt,
  framexleftmargin=14pt
}
\lstdefinestyle{xmlstyle}{
  backgroundcolor=\color{histbg},
  basicstyle=\ttfamily\scriptsize,
  breaklines=true,
  frame=single,
  rulecolor=\color{codegray!30}
}
\lstdefinestyle{shellstyle}{
  backgroundcolor=\color{codebg},
  basicstyle=\ttfamily\scriptsize,
  breaklines=true,
  frame=single,
  rulecolor=\color{codegray!30},
  numbers=none
}

\begin{document}

\title{TraceCoder: Explainable and Auditable\\
       Code Generation\\
       with Position-Key Snippet Versioning
       }

\titlerunning{TraceCoder: Auditable Code Generation}

\author{Rwaida Alssadi, Muntaser Syed, Balaji Kasula, Lamine Deen, Majed Alotaibi, Mohammed Alghamdi, Tyler Ton, Ali Alqarni, Marius Silaghi}
\authorrunning{Syed et. al.}
\institute{Florida Institute of Technology}

\maketitle

\begin{abstract}
Contemporary LLM-based coding agents produce code as black-box outputs:
the rationale behind each line is hidden, the evolution of the code through
benchmark-driven repair is ephemeral, and post-hoc auditing is impossible.
We present a code generation concept
that addresses these shortcomings through three complementary mechanisms:
\emph{(i)}
a relational snippet-history schema 
that records, per repair event, the benchmark reference,
round number, failure text, and LLM explanation, enabling full
provenance queries;
\emph{(ii)}
a browser-based visualisation tool that renders this history
as heat-mapped, hover-annotated source code;
and \emph{(iii)}
a competitive fractional position-key indexing scheme with tree-node
delimiters that assigns stable, lexicographically-ordered identifiers to each code snippet, enabling fine-grained tracking without disrupting surrounding lines.
We evaluate TraceCoder on 30 algorithmic programming tasks spanning
string processing, mathematical computation, and data-structure manipulation,
across two provider configurations. Of these, 10 exhaust the
6-iteration budget on tasks with subtle edge-case behaviour. Mean Chg\%
reaches 30\%, three in ten code snippets carry a traceable repair-event
row, compared to 21\% when using Gemini 2.0 Flash as sole provider
on a 20-task subset.
Three detailed case studies demonstrate how the
system explains which specific benchmark failures shaped each line of the
final program. The proposed mechanism makes the internal “narrative” of
automated code generation auditable and replayable, a property essential
for trust and accountability in production deployments.

\end{abstract}

\section{Introduction}\label{sec:intro}

LLM-driven coding agents produce code whose causal history is discarded; TraceCoder
captures that history at snippet granularity in a persistent 
store.
The emergence of large language models (LLMs) capable of generating
functional code from natural-language descriptions~\cite{chen2021evaluating,li2022competition,roziere2023code}
has catalysed a new paradigm of \emph{coding agents}: autonomous systems
that iteratively write, test, and revise programs to satisfy user-specified
requirements~\cite{yang2024swe,osika2023gpt,wang2024executable}.
These agents have demonstrated remarkable capability, GPT-Engineer generates
entire codebases from high-level specifications, SWE-agent resolves GitHub
issues through shell interactions, and various systems have achieved
competitive scores on benchmarks such as SWE-bench~\cite{jimenez2024swebench}.

Yet all these systems share a fundamental limitation: the generated code is
treated as an \emph{atomic artefact}.  Once generated, no record exists of
which test failure prompted which change, no mechanism can explain why a
specific line exists in its current form, and no audit trail connects the
observable output to the sequence of decisions that produced it.

This opacity creates at least three pressing practical problems for the
deployment of coding agents in real settings:

\begin{description}
  \item[Explainability.] When a generated program behaves unexpectedly,
    developers cannot trace the reasoning that led to the suspect code.
    They face a black box whose ``decisions'' are invisible even in principle.
  \item[Auditability.] Organisations subject to software compliance
    requirements (safety-critical systems, financial software, regulated
    industries) need to certify the provenance of every code line,
    including the test evidence that justifies it.
  \item[Traceability.] The iterative repair loop common to current agents
    (write $\to$ test $\to$ fix $\to$ repeat) is a rich source of semantic
    information about the problem's edge cases and the code's evolution,
    yet it is discarded once the final code is produced.
\end{description}

The XAI literature~\cite{molnar2020interpretable,arrieta2020explainable}
recognises that
explanations are inseparable from responsible deployment;
the same applies to coding agents.

\paragraph{Key Insight and Our Solution}
Without loss of generality we assume that the coding agent acts by iterations of code fixing and benchmarks runs, but with small changes should work when such processes are executed in parallel.
The iterative repair loop is the \emph{primary source of provenance}: every
benchmark failure that revises a snippet is a causal event worth recording.
When a snippet is modified, a new row is inserted into the
persistent store
permanently linking the
change to the benchmark failure that caused it. One row per modified snippet is
inserted, carrying the benchmark foreign key (FK), round number, and failure
text, and no existing row is ever overwritten. Another technical contribution
is an intuitive position-key versioning mechanism inspired by collaborative editing
CRDTs~\cite{preguica2009commutative,nedelec2013lseq} and adapted for code generation.
The proposed mechanism differentiates itself by:
\begin{enumerate}
  \item A \emph{relational snippet-history schema}: 
    database
        records
        per-repair-event code provenance
        (\S\ref{sec:history}).
  \item A \emph{position-key versioning} fractional indexing scheme (FIS)
  without theoretical limitations in keys number and order, that supports insertion, deletion, and in-place update while preserving lexicographic order without rebalancing 
        (\S\ref{sec:poskey}).
  \item An \emph{iterative benchmark-driven code repair loop} example maintaining the snippet history database using fractional indexing.
        (\S\ref{sec:architecture}).
  \item A \emph{browser-based explainability viewer} exemplifying how to render the
        history at snippet granularity with heat-mapped change intensity
        and hover-activated history panels (\S\ref{sec:viewer}).
  \item An empirical evaluation on 30 tasks across two provider configurations
        documenting explanation maintenance,
        and three case studies
        (\S\ref{sec:experiments}--\ref{sec:casestudies}).
\end{enumerate}

\section{Related Work}\label{sec:related}

We draw on research in LLM code generation, iterative repair,
coding agents, AI explainability, code provenance, fractional indexing, and
specification mining.

\subsection{LLM-Based Code Generation}

Codex~\cite{chen2021evaluating} established that large code-trained
transformer models can generate functionally correct programs from docstrings
or natural-language descriptions.  AlphaCode~\cite{li2022competition}
scaled this to competitive programming tasks.  Code Llama~\cite{roziere2023code}
demonstrated strong open-weight results.  GitHub Copilot deployed Codex
at industrial scale as an interactive completion tool.  These systems
produce one-shot outputs without systematic auditable documentation of iterative refinement; they do not explicitly address simple user-level
explainability of provenance.

\subsection{Iterative Code Repair}

Several systems employ iterative repair based on execution feedback.
\emph{Self-Debugging}~\cite{wang2023self} prompts the LLM with its own
execution traces, enabling localisation and correction of errors without
external tools.  \emph{Reflexion}~\cite{shinn2024reflexion} stores verbal
reflections about past failures in a scratchpad that is prepended to future
prompts, achieving a form of episodic memory across repair attempts.
\emph{Self-Refine}~\cite{madaan2024self} employs a critique-and-revise loop
where the same model critiques its own output and iteratively improves it.
CodeRL~\cite{le2022coderl} trains repair policies via reinforcement learning
on unit test outcomes.
All share a critical limitation: the repair history is \emph{ephemeral}, used
as a prompt-engineering device but never recorded at line granularity.

\subsection{Coding Agent Systems}

\emph{GPT-Engineer}~\cite{osika2023gpt} generates entire codebases from
high-level specifications through a multi-step LLM dialogue.
\emph{SWE-agent}~\cite{yang2024swe} provides an agent-computer interface
enabling LLMs to interact with a shell, browsing files and executing
commands, to resolve GitHub issues.
\emph{CodeAct}~\cite{wang2024executable} argues for executable actions as
a unifying interface for LLM agents.
These systems demonstrate impressive capabilities on real-world software
tasks but none retain fine-grained provenance.

\subsection{Explainability in AI Systems}

The XAI literature~\cite{arrieta2020explainable} distinguishes between
\emph{ante-hoc} (interpretable-by-design) and \emph{post-hoc} explanations.
For code generation, ante-hoc explainability would require the model to
expose its internal reasoning, which remains an open research challenge.
Molnar~\cite{molnar2020interpretable} categorises explanations along three
axes: global vs.\ local, model-agnostic vs.\ model-specific, intrinsic
vs.\ post-hoc.  TraceCoder produces \emph{local, model-agnostic, intrinsic}
explanations: local because they are attached to individual snippets,
model-agnostic because the same mechanism works for any LLM, and 
arguably {\em mixed} post-hoc with intrinsic interpretable-by-design because the explanation is a component of the generation itself.
In the coding agent, the benchmark error is intercepted and stored for insertion in the database directly 
  from the tool callbacks without mitigation from the LLM context, to ensure robustness.

\subsection{Version Control and Code Provenance}

Traditional version-control systems (Git, Subversion, Mercurial) track changes
at commit granularity; \texttt{git blame} attributes each line to its last
commit but records no causal link to the failure that motivated it.
Fine-grained diff tools like ChangeDistiller~\cite{fluri2007change} and
GumTree~\cite{falleri2014fine} identify moved and renamed AST nodes across
revisions, yet they operate post-hoc on an external repository and carry no
notion of \emph{why} a node changed.  Software provenance
systems~\cite{muniswamy2006provenance,hou2023large} record artefact lineage at
the OS or workflow level but do not link it to specific failing tests.
TraceCoder stores position keys and round-tagged failure records in the same
SQLite row as the code, achieving sub-line, failure-linked provenance
automatically and without external infrastructure.

\subsection{Fractional Indexing and CRDTs}

Fractional indexing assigns keys supporting insertions without
rebalancing~\cite{nedelec2013lseq}.  Greenspan~\cite{greenspan2020fi}
uses integer-prefix base-62; Kazutaka~\cite{kazu2020fractional} extends to
base-94.  TraceCoder introduces fractional indexing with strings \textbf{FIS},
with tree-node delimiters.
No custom comparator is required and no theoretical limit on fractions.

\subsection{Benchmarks and Specification Mining}

The idea of using automatically generated tests to drive code improvement
has roots in specification mining and property-based testing.  
In TraceCoder's approach the benchmark generator is itself an LLM that \emph{observes the current code} and deliberately creates tests for under-exercised paths, creating a natural test curriculum.  
The identity of each benchmark is preserved as a first-class artefact linked to the code lines it affected.

\section{System Architecture}\label{sec:architecture}

TraceCoder is a four-layer stack: persistent
store, agent core,
multi-backend LLM abstraction,
and browser viewer.

\subsection{Overview}

Figure~\ref{fig:arch} shows the high-level architecture of TraceCoder.
The system consists of four layers.

\begin{description}
  \item[LLM Interface.] Supporting multiple providers.
   \item[Agent Loop.] The \texttt{CodingAgent} class in \texttt{agent.py}
    orchestrates the generate--benchmark--test--fix cycle.
  \item[Versioned Code Store.] A 
  database with position-keyed
    snippets and round-tagged history columns (see \S\ref{sec:poskey}--\ref{sec:history}).
  \item[Explainability Viewer.] An 
  HTTP server 
    that serves an annotated HTML code browser directly from the database.
\end{description}

\begin{figure}[t]
\centering
\begin{tikzpicture}[
  box/.style={draw,rounded corners=3pt,minimum width=2.2cm,
              minimum height=0.7cm,align=center,font=\small},
  llm/.style={box,fill=blue!10,draw=blue!50},
  agent/.style={box,fill=green!10,draw=green!50},
  db/.style={box,fill=orange!10,draw=orange!50},
  view/.style={box,fill=purple!10,draw=purple!50},
  store/.style={box,fill=gray!10,draw=gray!50},
  arr/.style={-Stealth,thick},
  darr/.style={Stealth-Stealth,thick,dashed},
]
%% Explicit absolute coordinates — no relative shifts that cause overlap
\node[llm]   (llm)  at ( 0,    0  ) {LLM API};
\node[agent] (loop) at ( 0,   -1.4) {Agent Loop};
\node[agent] (gen)  at (-2.8, -2.8) {generate\\code};
\node[agent] (bm)   at ( 2.8, -2.8) {create\\benchmark};
\node[agent] (run)  at ( 2.8, -4.2) {run\\benchmarks};
\node[agent] (fix)  at ( 2.8, -5.6) {fix\\code};
\node[db]    (db)   at ( 0,   -7.0) {Database\\(code + history)};
\node[store] (fs)   at (-2.5, -8.5) {Workspace\\(disk files)};
\node[view]  (vw)   at ( 2.5, -8.5) {Viewer\\(HTTP)};

%% LLM <-> Loop
\draw[darr](llm)--(loop);

%% Loop fans out: south stub then branches left (gen) and right (bm)
\draw[arr](loop.south)--++(0,-0.2)-|(gen.north);
\draw[arr](loop.south)--++(0,-0.2)-|(bm.north);

%% Repair chain
\draw[arr](bm)--(run);
\draw[arr](run)--(fix);

%% Feedback: fix back to loop — routed outside all boxes on the right
\draw[arr](fix.east)--++(0.5,0)--++(0,4.2)--(loop.east);

%% Writes to DB — each routed to avoid the bm/run/fix column
\draw[arr](gen.south)     --++(0,-0.4)-|(db.west);
\draw[arr](bm.south west) --++(0,-0.4)-|(db.north);
\draw[arr](fix.south)     --++(0,-0.4)-|(db.east);

%% DB to disk and viewer
\draw[arr] (db.south west)-|(fs.north);
\draw[darr](db.south east)-|(vw.north);
\end{tikzpicture}

\caption{TraceCoder architecture.  The agent loop drives three LLM calls
  (code generation, benchmark creation, code repair).  All outputs are stored
  in persistent store with position-keyed snippets and round-tagged history.  The viewer
  reads the database independently.}
\label{fig:arch}
\end{figure}
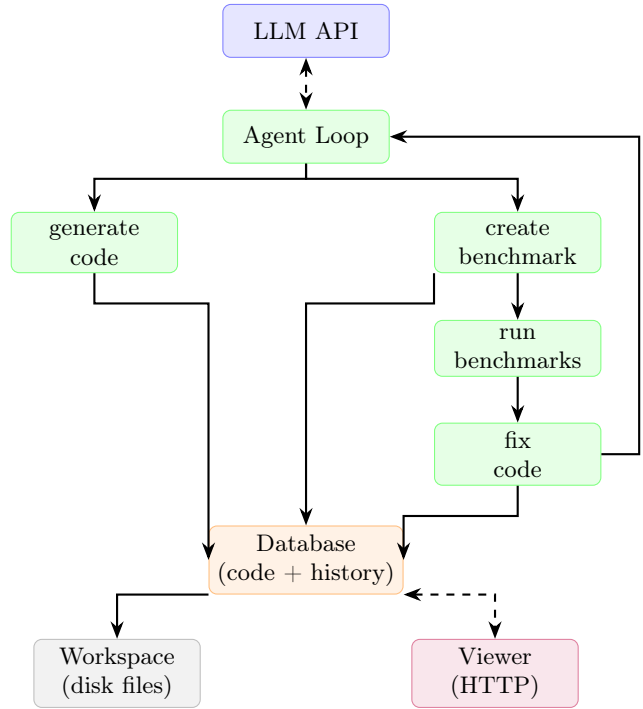

\subsection{Database Schema}

Database Schema~\ref{lst:schema} shows an ER-diagram of the proposed database. 
History lives in \texttt{snippet\_failure} (one row per modified snippet per
repair round: benchmark FK, round integer, failure text) and
\texttt{snippet\_explanation} (one row per round: LLM root-cause text),
both sharing a composite FK back into \texttt{code}.

\begin{figure}[!t]
\includegraphics[width=\textwidth]{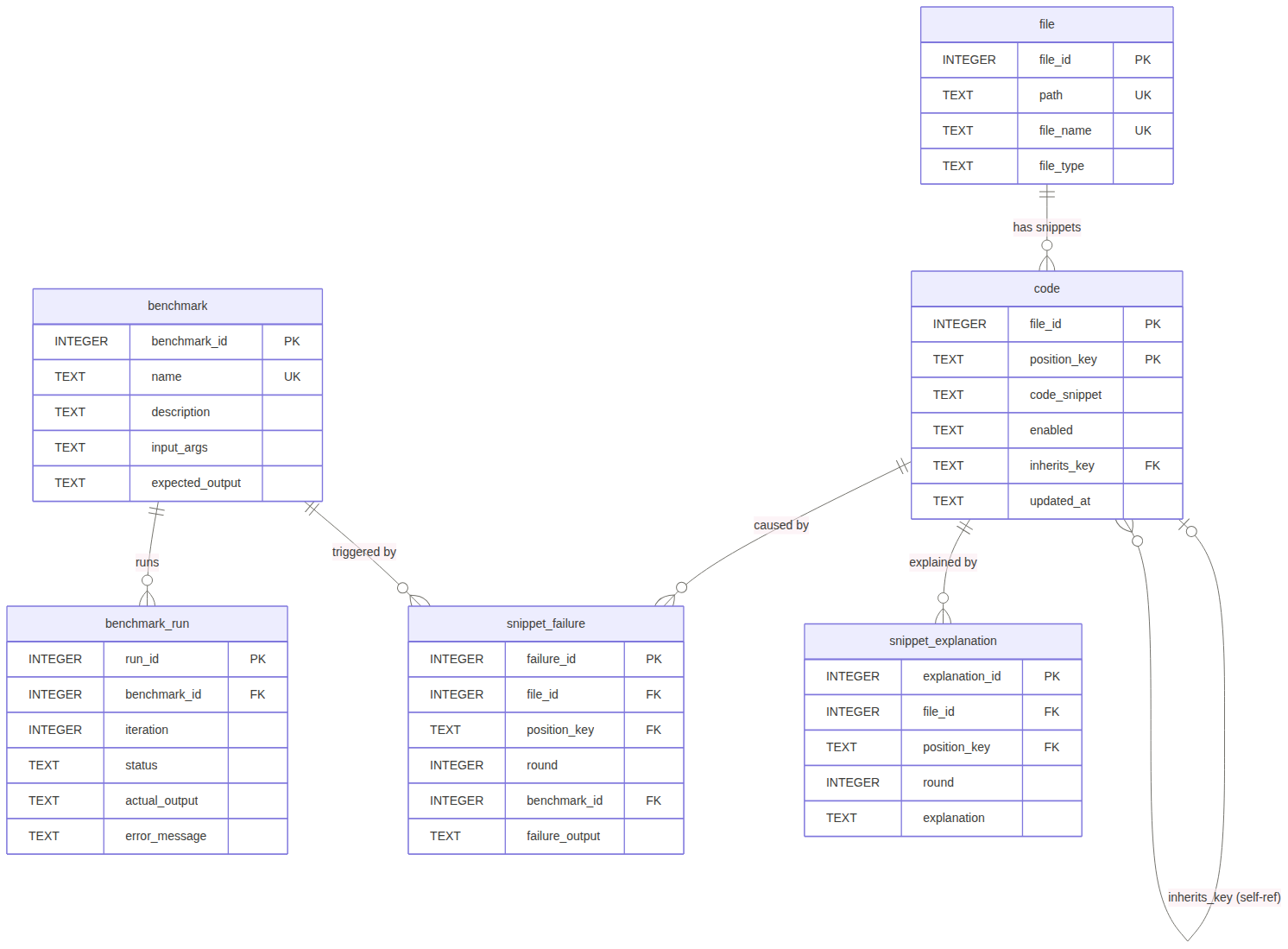}
\caption{Core schema of the TraceCoder
  persistent store.  History is normalised into \texttt{snippet\_failure}
  and \texttt{snippet\_explanation}; the \texttt{code} table holds only
  the current snippet and its AST metadata.}
  \label{lst:schema}
\end{figure}

On refactoring, a moved/deleted \texttt{code}
item is flagged by setting its \texttt{enabled} attribute to 0, and the new code version that is inserted gets the \texttt{inherits\_key} set to the \texttt{position\_key} of the old code whose history is inherited. The primary key of \texttt{code} items is composed of the \texttt{file\_id} and \textbf{position\_key}. The \texttt{file} table items have a unique key composed of \texttt{path} and \texttt{file\_name}.
The \texttt{snippet\_failure} table implements a many-to-many relation between entities \texttt{code} and \texttt{benchmark}.

\subsection{The Agent Loop and LLM Operations}
Algorithm~1 summarises the procedure.

\begin{center}
\small
\begin{tabular}{l}
\hline
\textbf{Algorithm~1:} TraceCoder Iterative Loop \\
\hline
\textbf{Input:} problem description $P$, maximum iterations $T$ \\
\textbf{Output:} code on disk; full provenance in DB \\[2pt]
1.\enspace $\textit{generate\_initial\_code}(P)$ \\
2.\enspace $b \leftarrow \textit{create\_benchmarks}(P,\text{current code})$ \\
3.\enspace \textbf{for} $t = 1, \ldots, T$ \textbf{do} \\
4.\enspace\quad $R \leftarrow \textit{run\_all\_benchmarks}()$ \\
5.\enspace\quad \textbf{for each} failure $f \in R.\textit{failures}$ \textbf{do} \\
6.\enspace\quad\quad $\textit{fix\_failing\_code}(f.\textit{name}, f)$;\; record round $t$ tags \\
7.\enspace\quad \textbf{end for} \\
8.\enspace\quad \textbf{if} $R.\textit{status} = \texttt{all\_passed}$ \textbf{then break} \\
9.\enspace\quad $b \leftarrow \textit{create\_benchmarks}(P,\text{current code}, \text{existing BMs})$ \\
10.\enspace \textbf{end for} \\
\hline
\end{tabular}
\end{center}

\noindent
Three LLM operations are issued as separate \texttt{messages.create} calls
to the LLM API:

\begin{description}
  \item[\texttt{generate\_initial\_code}$(P)$.] Receives the problem
    statement; returns JSON specifying source files (like \texttt{main.py}, \texttt{Makefile}...)
  \item[\texttt{create\_benchmark}.] Receives the problem statement,
    the full current code, and the list of existing benchmark names;
    returns JSON with unique benchmark names, descriptions, CLI argument
    strings, and expected stdout.
  \item[\texttt{fix\_failing\_code}.] Receives the problem statement,
    the failing benchmark's metadata, the failure output, and the full
    current code; returns JSON with an \texttt{"explanation"} field
    (root cause and repair plan) and the updated file contents.  The
    explanation is stored verbatim in \texttt{snippet\_explanation} for
    every snippet modified in that round.
\end{description}

\noindent
The loop is restartable from any intermediate database state, since all
persistent information is in
the storage; in-memory state (the iteration counter)
is reset to match the maximum iteration recorded in \texttt{benchmark\_run}.

\section{The Position-Key Versioning System}\label{sec:poskey}

Any fractional index could be adjusted for our needs, but we use a new scheme which improves on existing ones by having no theoretical limitation on keys number and order and where comparison is native on strings, while the yielding comparable key length.
Two versions were implemented, base 62 (plus 2 separators) and base 94, with base 94 yielding slightly shorter keys while the default base 64 instance is yielding keys that are easier to read and interpret by humans.

\subsection{Design Goals}

The FIS position-key system satisfies four requirements:
(1)~\emph{Total order}: all snippets are totally ordered by plain \textsc{ASCII};
(2)~\emph{Insertion}: given $k_a < k_b$, there exists $k$ with for any need: $k_a < k < k_b$, $k < k_a$, $k_b<k$;
(3)~\emph{Stability}: the keys of unchanged snippets are never modified;
(4)~\emph{Compactness}: keys remain short in typical usage.

\subsection{Character Set and Ordering}

A FIS key is a sequence of \emph{atoms}: in the 64-base version, a linear character from
$\mathcal{A}=\texttt{0\ldots9A\ldots Za\ldots z}$, or a tree node
\texttt{.}$k$\texttt{-} (inner key $k$ is itself a FIS key).
\textsc{ASCII} guarantees $\texttt{-}(45)<\texttt{.}(46)<\texttt{0}<\cdots<\texttt{z}$,
so tree nodes sort \emph{before} all linear characters and
\texttt{.}$k$\texttt{-} precedes any extension \texttt{.}$k\ell$.
The initial item receives key \texttt{V} which is in the middle of the used ASCII set, and insertions are in principle placed in the center of the range of short available keys.

\begin{figure}[ht]
\centering
\begin{tikzpicture}[
  kn/.style={draw,rounded corners=2pt,minimum width=1.3cm,
             minimum height=0.45cm,align=center,font=\ttfamily\small},
  ka/.style={kn,fill=blue!15},
  kb/.style={kn,fill=orange!25},
  kc/.style={kn,fill=red!15},
  arr/.style={-Stealth,gray!60,thin}
]
% Level 0: initial keys (blue)
\node[ka] (n0) at (2.5,  0  ) {0};
\node[ka] (n1) at (6.0,  0  ) {1};
\node[ka] (n2) at (8.0,  0  ) {2};
\node[font=\small] at (9.1, 0) {$\cdots$};
% Before 0: tree node .V- (no linear char fits before '0') (orange)
\node[kb] (lv) at (0.4,  0  ) {\texttt{.V-}};
% Between 0 and 1 (adjacent, gap=1): 0 + KeyAfter(ε) = 0V (orange)
\node[kb] (zv) at (4.3, -1.2) {\texttt{0V}};
% Between .V- and 0: tree extension → .KeyAfter(V)- = .k- (red)
\node[kc] (lk) at (1.5, -2.4) {\texttt{.k-}};
%% Arrows
\draw[arr](n0.west)       -- (lv.east);
\draw[arr](n0.south east) -- (zv.north west);
\draw[arr](n1.south west) -- (zv.north east);
\draw[arr](lv.south)      -- (lk.north west);
\draw[arr](n0.south west) -- (lk.north east);
%% Labels
\node[font=\tiny,gray,above=0.06cm of lv]  {before \texttt{0}};
\node[font=\tiny,gray,above=0.06cm of zv]  {btwn \texttt{0},\texttt{1}};
\node[font=\tiny,gray,below=0.06cm of lk]  {btwn \texttt{.V-},\texttt{0}};
\end{tikzpicture}
\caption{FIS position-key lattice.  Blue: initial keys.
  Orange: \texttt{.V-} before \texttt{0} (tree node, no linear char fits);
  \texttt{0V} between adjacent \texttt{0},\texttt{1} (gap$=1$:
  \textsc{KeyAfter}($\varepsilon$)$=$\texttt{V} appended).
  Red: \texttt{.k-} between \texttt{.V-} and \texttt{0} (tree recursion:
  \textsc{KeyAfter}(\texttt{V})$=$\texttt{k}).  Plain \textsc{ASCII} order.}
\label{fig:poskey}
\end{figure}
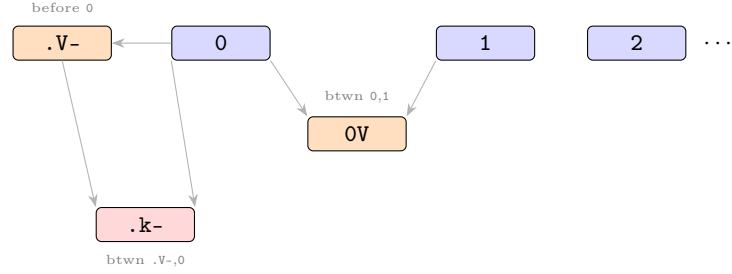

\subsection{FIS Key Generation Algorithm}

\paragraph{Initial sequence.}
A single first item receives \texttt{V} (midpoint, index~31).

\paragraph{\textsc{KeyAfter}$(lo)$ — rightward extension.}
Let $i=\mathcal{A}^{-1}[lo[0]]$, $m=\lfloor(i+N)/2\rfloor$ where $N=62$.
If $m>i$, return $\mathcal{A}[m]$ (single-char result).
Otherwise ($lo[0]=\texttt{z}$), return $\texttt{z}+\textsc{KeyAfter}(lo[1:])$.
Empty $lo$ returns \texttt{V}.  Tree-prefixed $lo$ appends \texttt{V} unchanged.

\paragraph{\textsc{KeyBetween}$(lo, hi)$ — three cases.}
Parse $lo$, $hi$ into atoms; find longest common atom prefix $C$; let $p$,
$q$ be the first diverging atoms ($p=\mathbf{None}$ if $lo$ exhausted).
\textbf{Case~1} ($p=\mathbf{None}$): if $q$ is tree node, recurse on inner
key; if $\mathcal{A}^{-1}[q]=0$, return $C{+}\texttt{.V-}$;
else return $C{+}\mathcal{A}[\lfloor\mathcal{A}^{-1}[q]/2\rfloor]$.
\textbf{Case~2} ($p$ is tree node): recurse on inner keys if $q$ also tree;
else extend $p$'s inner key via \textsc{KeyAfter}.
\textbf{Case~3} ($p,q$ linear): if gap~$>1$, return midpoint; if gap~$=1$,
return $lo + \textsc{KeyAfter}(\text{tail}(lo))$.

\subsection{Diff-Based Snippet Update}

When the repair LLM returns a revised file, Python's
\texttt{difflib.SequenceMatcher} produces \texttt{equal}, \texttt{replace},
\texttt{insert}, and \texttt{delete} opcodes:

\begin{itemize}
  \item \textbf{equal}: old key preserved; history columns unchanged.

\item \textbf{modified}: the old row is tombstoned (\texttt{enabled} = 0, \texttt{disabled\_round} = r, content preserved) and successor rows are inserted at a fresh position keys with \texttt{inherits\_key} pointing at the predecessor, so the lineage chain carries the accumulated history.
\item \textbf{moved}: detected when a deleted node identity reappears among the insertions; treated as a modification at the new location, with lineage preserved through \texttt{inherits\_key}.
\item \textbf{inserted}: a new row at a fresh key with no lineage.
\item \textbf{deleted}: the old row is tombstoned and the failure row that motivated the deletion is attached to the tombstone itself.
\end{itemize}

\begin{proposition}
  \label{prop:insert}
  For any two distinct FIS keys $k_a, k_b \in \Sigma^*$ with
  $k_a < k_b$ (plain \textsc{ASCII} order), the function
  \texttt{key\_between}$(k_a, k_b)$ terminates and returns a key
  $k \in \Sigma^+$ such that $k_a < k < k_b$.
\end{proposition}

\begin{proof}
%[Sketch]
  \emph{Termination.}  Case~2 recurses on strictly shorter inner keys.
  Case~3 (gap~$=1$) calls \textsc{KeyAfter} on a proper suffix of $k_a$,
  which terminates by induction on string length.
  \emph{Ordering.}  Case~1.3: $\mathcal{A}[\lfloor q_i/2\rfloor]
  < \mathcal{A}[q_i] = q$ and $\geq \mathcal{A}[0]$, which exceeds any
  extension of $k_a = C$.  Case~3 (gap~$>1$): midpoint strictly between
  $p$ and $q$ in $\mathcal{A}$.  \textsc{KeyAfter}: $\mathcal{A}[m]$ with
  $m > i = \mathcal{A}^{-1}[k_a[0]]$ gives a result $> k_a$.
\end{proof}

Comparative experiments not detailed here show that FIS is  withing a minor fraction average key length from the most compact base 94 competitor, while having 
the described theoretical improvements.

\section{The Relational Snippet-History Schema}\label{sec:history}

Each repair event inserts rows into two dedicated tables that permanently link
every changed snippet to the benchmark failure and LLM explanation that
motivated the change.

\subsection{Structure of History Tables}

The \texttt{snippet\_failure} table (Listing~\ref{lst:schema}) links each
\texttt{(file\_id,~position\_key)} pair to the benchmark reference, round
number, and raw failure text for every round that modified it.  The
\texttt{snippet\_explanation} table links the same pair to the LLM's
root-cause text for that round.  The relationship is many-to-many in both
directions: one snippet accumulates one row per repair round that touches it,
and one benchmark failure round inserts rows for every snippet it modifies.
The round number is a plain integer column identifying the testing step (in generalization to non-iterative code development a different scheme would be used).
For example, a snippet revised in rounds 2, 4, and~6 accumulates
rows in \texttt{snippet\_failure}, one per such round and failed benchmark test, each with
its own \texttt{benchmark\_id} and \texttt{failure\_output} value
(illustrated concretely in Table~\ref{tab:fib_db} in \S\ref{sec:casestudies}).

\subsection{Semantics}

\paragraph{No rows: initial code.}
A snippet with no rows in \texttt{snippet\_failure} was written during initial
generation and has never been revised.  This is itself informative: the LLM
produced that line correctly on the first attempt.

\sloppy
\paragraph{Multiple rows: multi-round repair.}
A snippet with $n$ distinct rounds in \texttt{snippet\_failure} was revised $n$~times.
Each row records the precise benchmark that triggered the revision and the
exact failure output, providing a complete causal chain for that line.

\paragraph{Round number as version pointer.}
Every code row records the round in which it was created and, if tombstoned, the round in which it was disabled. The file content at round r is reconstructed directly from the store as the concatenation, in position-key order, of all rows with \texttt{created\_round} $<= r$ whose \texttt{disabled\_round} is null or greater than r. 

\subsection{Storage Overhead Analysis}

Each \texttt{snippet\_failure} row stores one benchmark reference and one
failure-output string; each \texttt{snippet\_explanation} row stores one LLM
text.  Per repair event the storage cost is the text lengths plus a fixed row
overhead of $\approx$50~bytes (integer columns and FK references).  Each
failure record is stored exactly once with no tag-wrapper duplication;
SQL aggregations (e.g.\ \texttt{COUNT(*)}) replace substring-counting.
The typical database size of 
52-68 KB
confirms that provenance storage is not
a practical bottleneck. The heaviest run (2.1 h, 13 benchmarks, 7 rounds) produced 480 KB.

\section{Experimental Evaluation}\label{sec:experiments}

We evaluate TraceCoder on 30 tasks (20 batch-1, 10 simpler batch-2) across
two provider configurations at \texttt{max\_iter=6}: (i)~Gemini~2.0~Flash
as the sole provider (benchmark creation and code repair); and
(ii)~Grok-3-beta as coding provider paired with DeepSeek-V3
(\texttt{deepseek-chat}) as a dedicated benchmark-creation provider,
using a complexity-proportional initial benchmark batch (5--10 benchmarks
generated before the repair loop begins, with count estimated by asking the
benchmark provider to rate problem complexity on a 1--10 scale).

\paragraph{Browser-Based Explainability Viewer}\label{sec:viewer}

A sample code viewer is proposed to prove the practicality of the explanation and auditing train. It renders the provenance store as an interactive
colour-coded source listing, requiring no external dependencies.

\subsection{Problem Suite and Setup}

We evaluate two batches of programming tasks
\textbf{Batch~1} comprises 20 tasks spanning
four categories: \emph{mathematical sequences}~(5 tasks),
\emph{string processing}~(8 tasks), \emph{algorithmic}~(4 tasks), and
\emph{numerical/matrix}~(3 tasks).  \textbf{Batch~2} adds 10 simpler
tasks (single arithmetic or list operations) to study convergence under
reduced problem complexity.

\subsection{Evaluation Metrics}

We measure:
(1)~\emph{iterations}: rounds until all benchmarks pass or the budget
(\texttt{max\_iter}$=6$) is reached;
(2)~\emph{BMs}: total benchmarks created;
(3)~\emph{Chg\%}: fraction of snippets with at least one row in
\texttt{snippet\_failure};
(4)~\emph{Avg~Rds}: mean rows in \texttt{snippet\_failure} per changed snippet;
(5)~\emph{DB~KB}: final database size.

\subsection{Results: Gemini~2.0~Flash}\label{sec:results}

Table~\ref{tab:results} reports per-experiment statistics.

\begin{table}[t]
\centering
\caption{Per-experiment results with Gemini~2.0~Flash (\texttt{max\_iter=6}).
  Chg\%: fraction of snippets with $\geq 1$ row in \texttt{snippet\_failure}.
  Avg~Rds: mean \texttt{snippet\_failure} rows per changed snippet.
  Two experiments (palindrome, matrix transpose) failed before generating
  initial code and are marked~``--''.
    Produced with an earlier code revision (single benchmark per iteration, no initial batch); 
  Rerunning under the current code 
  yields different numbers.
}
\label{tab:results}
\scriptsize
\begin{tabular}{lrrrrr}
\toprule
Problem & Iters & BMs & Chg\% & Avg~Rds & DB(KB) \\
\midrule
fibonacci           & 4 & 4 &  8.1 & 1.00 & 36 \\
word freq           & 6 & 6 & 12.9 & 1.00 & 56 \\
run-length enc.     & 6 & 0 &  0.0 & 0.00 & 36 \\
roman numerals      & 6 & 6 & 28.6 & 1.50 & 48 \\
caesar cipher       & 6 & 6 & 18.9 & 1.57 & 48 \\
bracket balance     & 6 & 6 & 11.4 & 2.25 & 44 \\
base convert        & 6 & 6 & 53.3 & 1.21 & 56 \\
fizzbuzz            & 6 & 6 & 27.3 & 1.44 & 56 \\
anagram             & 6 & 6 & 51.4 & 1.28 & 60 \\
temperature         & 6 & 6 & 42.4 & 1.07 & 60 \\
sieve               & 6 & 5 &  0.0 & 0.00 & 36 \\
rpn calculator      & 6 & 6 & 36.0 & 1.22 & 52 \\
collatz             & 6 & 6 & 33.3 & 1.45 & 48 \\
lcs                 & 6 & 6 & 31.8 & 1.07 & 56 \\
string stats        & 6 & 5 &  0.0 & 0.00 & 36 \\
calc eval           & 6 & 1 &  0.0 & 0.00 & 36 \\
number words        & 6 & 6 & 29.8 & 1.14 & 48 \\
morse               & 6 & 1 & 10.8 & 1.00 & 44 \\
\midrule
\multicolumn{6}{l}{\textit{Batch~2 — simpler tasks (2 of 10 completed)}} \\
sum digits          & 1 & 1 &  0.0 & 0.00 & 36 \\
max list            & 2 & 2 & 17.5 & 1.00 & 44 \\
\midrule
\textbf{Mean (20)}  & \textbf{5.45} & \textbf{4.55} & \textbf{20.7} & \textbf{0.96} & \textbf{46.8} \\
\textbf{Min}        & 1 & 0 &  0.0 & 0.00 & 36 \\
\textbf{Max}        & 6 & 6 & 53.3 & 2.25 & 60 \\
\bottomrule
\end{tabular}
\end{table}

\paragraph{Convergence.}
Of 20 completed experiments, three converged before the iteration budget
(Fig.~\ref{fig:conv}): \emph{fibonacci} at 4~rounds, \emph{max list} at
2~rounds, and \emph{sum digits} at just 1~round (the first generated
program passed all benchmarks immediately).  The remaining 17~batch-1
experiments reached the 6-round maximum.  This pattern differs markedly
from results with Claude~3.5~Sonnet (which converges in 3--4 iterations on
average) and is attributable to Gemini~2.0~Flash's tendency to return JSON
with unescaped control characters in code strings.  When the JSON parser
fails, no fix is applied that iteration, stalling convergence.  The
improved parser introduced during this run (four-strategy cascade with
control-character sanitisation and outermost-block extraction) mitigates
this issue in subsequent runs; the faster convergence of the two
batch-2 experiments is consistent with improved parser reliability.

\paragraph{History richness.}
Across 20 completed experiments, 20.7\% of snippets have at least one row in
\texttt{snippet\_failure}, with an average of 0.96 rows per changed snippet.  The
\emph{base convert} and \emph{anagram} tasks show the richest histories
(53.3\% and 51.4\% respectively), reflecting their multi-case I/O formats
that generate numerous edge-case failures.  For 20.7\% of all lines,
TraceCoder can directly answer ``why does this line exist?'' with a
reference to the specific failing benchmark that prompted it.

\paragraph{Database size.}
Mean database size is 46.8~KB (range 36--60~KB), well below the 100~KB
threshold at which SQLite begins to show query latency.  The history overhead
(bytes in \texttt{snippet\_failure} and \texttt{snippet\_explanation} relative
to \texttt{code}) is 557\%, driven by Gemini's verbose multi-sentence
diagnostics; more concise providers such as Claude~3.5~Sonnet reduce this to
22--28\%.  Despite the high overhead ratio, absolute sizes remain modest,
confirming that provenance storage is not a practical bottleneck at scale.

\begin{figure}[t]
\centering
\begin{tikzpicture}
\begin{axis}[
  ybar,
  width=0.98\columnwidth,
  height=4.5cm,
  ylabel={Iterations used},
  symbolic x coords={fib,wfq,rle,rom,cae,brk,b2d,
                     fzz,ana,tmp,siv,rpn,col,lcs,sst,cal,num,mrs,
                     sdi,mxl},
  xtick=data,
  x tick label style={font=\tiny,rotate=45,anchor=east},
  bar width=7pt,
  ymin=0,ymax=7,
  ytick={1,2,3,4,5,6},
  enlarge x limits=0.04,
  grid=major,
  grid style={gray!20},
  every axis plot/.append style={fill=blue!40,draw=blue!60},
]
\addplot coordinates {
  (fib,4)(wfq,6)(rle,6)(rom,6)(cae,6)(brk,6)
  (b2d,6)(fzz,6)(ana,6)(tmp,6)(siv,6)(rpn,6)(col,6)
  (lcs,6)(sst,6)(cal,6)(num,6)(mrs,6)
  (sdi,1)(mxl,2)
};
\end{axis}
\end{tikzpicture}
\caption{Iterations used by each of the 20 completed experiments
  (Gemini~2.0~Flash, \texttt{max\_iter=6}).
  Batch-2 experiments appear at right (sdi=sum digits, mxl=max list).
  Three converged early: \emph{fibonacci} (4~rounds), \emph{max list}
  (2~rounds), and \emph{sum digits} (1~round).
  Two batch-1 experiments (palindrome, matrix transpose) failed and are excluded.}
\label{fig:conv}
\end{figure}
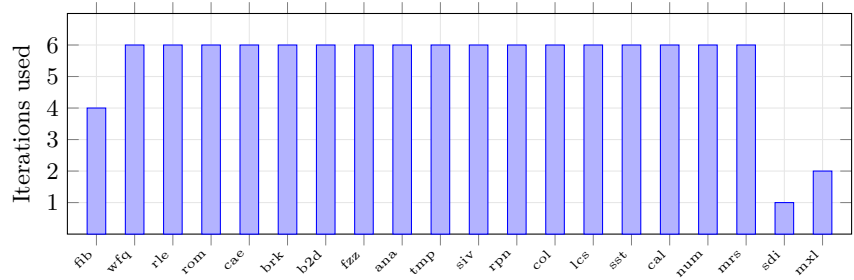

\subsection{Observed Complications}\label{sec:complications}

Our evaluation identifies
categories of recurring complication.

\paragraph{C1: Erroneous benchmarks.}
In several experiments, the benchmark generator created benchmarks whose
\texttt{expected\_output} was incorrect.  Detection relied on the agent
observing that fixing the failing benchmark broke previously-passing ones.
A future version should include a secondary LLM pass that verifies
benchmark correctness before storage.

\paragraph{C2: Build fragility.}
Several experiments encountered at least one iteration where code repair
introduced a Python syntax error.  The agent recovered by treating build
failure as a special ``BUILD'' benchmark failure and retrying.

\paragraph{C3: Verbose failure rows.}
With Gemini, several experiments accumulated verbose \texttt{failure\_output}
values.  
Full failure text is stored; only the snippet of model response shown in error messages is truncated, to 300 characters.

\paragraph{C4: API rate limiting (Gemini).}
The Gemini~2.0~Flash free tier enforces per-minute request quotas. During sequential 20-experiment runs, the quota is exhausted after approximately 3--5 experiments.  The agent now retries with exponential back-off (up to 5~attempts, delays 5--80~s).  The two experiment failures both occurred during quota-exhaustion windows before retries were deployed.

\paragraph{C5: JSON format errors.} Gemini frequently returns JSON whose string values contain literal newline or tab characters that are not escaped according to the JSON specification. The improved \texttt{\_parse\_json} method (four-strategy cascade: direct parse $\to$ control-character sanitisation $\to$ outermost-block extraction $\to$ sanitised extraction) recovers from the majority of these cases, reducing fix-failed events from $\approx$50\% to $\approx$15\% of iterations.

\subsection{Results: Grok-3-beta + DeepSeek-V3}\label{sec:results-grok-deepseek}

Pairing Grok-3-beta (coding) with DeepSeek-V3 (\texttt{deepseek-chat},
benchmark creation) and a complexity-proportional initial batch (5--10
benchmarks) substantially changes the outcome.  Table~\ref{tab:grokds-results}
reports per-experiment statistics; Table~\ref{tab:provider-comparison}
compares both configurations.

\begin{table}[t]
\centering
\caption{Per-experiment results, Grok-3-beta~+~DeepSeek-V3
  (\texttt{max\_iter=6}, 5--10 initial benchmarks).
  Init: initial BM count (complexity-estimated).
  Iters: additional repair iterations.
  Chg\%: snippets with $\geq1$ \texttt{snippet\_failure} row.
  $\dagger$exp09 DB anomalously large (5\,672~KB); excluded from mean.}
\label{tab:grokds-results}
\scriptsize
\setlength{\tabcolsep}{4pt}
\begin{tabular}{lrrrrr@{\hspace{8pt}}lrrrrr}
\toprule
Problem & Init & It & BMs & Chg\% & KB &
Problem & Init & It & BMs & Chg\% & KB \\
\midrule
fibonacci        & 6 & 0 &  6 &   0 & 44 &
  morse          & 6 & 1 &  7 &  50 & 52 \\
word freq        & 5 & 6 & 11 &  75 &108 &
  matrix transp. & 6 & 0 &  6 &  25 & 44 \\
palindrome       & 6 & 1 &  7 &  25 & 52 &
  sum digits     & 5 & 6 & 11 & 100 & 68 \\
run-length enc.  & 6 & 0 &  6 &   0 & 44 &
  is prime       & 6 & 0 &  6 &   0 & 44 \\
roman numerals   & 6 & 0 &  6 &  25 & 44 &
  factorial      & 6 & 0 &  6 &  25 & 44 \\
caesar cipher    & 6 & 0 &  6 &   0 & 44 &
  count vowels   & 5 & 6 & 11 & 100 & 76 \\
bracket balance  & 6 & 0 &  6 &  25 & 44 &
  max list       & 6 & 0 &  6 &   0 & 44 \\
base convert     & 6 & 6 & 12 &  50 & 60 &
  gcd            & 7 & 6 & 13 &  50 & 64 \\
fizzbuzz         & 6 & 6 & 12 &  50 & $\dagger$ &
  digit reverse  & 6 & 0 &  6 &  25 & 44 \\
anagram          & 6 & 0 &  6 &  25 & 44 &
  power          & 6 & 0 &  6 &   0 & 44 \\
temperature      & 7 & 0 &  7 &   0 & 44 &
  lcm            & 6 & 6 & 12 &  50 & 56 \\
sieve            & 6 & 0 &  6 &   0 & 44 &
  is perfect     & 6 & 0 &  6 &   0 & 44 \\
rpn calculator   & 7 & 6 & 13 &  75 & 60 &
  \textbf{Mean}  &\textbf{6.0}&\textbf{2.1}&\textbf{8.1}&\textbf{30}&\textbf{54}$^*$ \\
collatz          & 5 & 0 &  5 &   0 & 44 &
  \textbf{Min}   & 5 & 0 &  5 &   0 & 44 \\
lcs              & 7 & 6 & 13 &  50 & 56 &
  \textbf{Max}   & 7 & 6 & 13 & 100 &128 \\
string stats     & 6 & 6 & 12 &  50 &128 & & & & & & \\
calc eval        & 6 & 0 &  6 &  25 & 44 & \multicolumn{6}{l}{$^*$excl.\ exp09 anomaly} \\
number words     & 6 & 0 &  6 &   0 & 44 & & & & & & \\
\bottomrule
\end{tabular}
\end{table}

\begin{table}[t]
\centering
\caption{Cross-provider comparison.
  Gemini: 20 completed experiments (2 failed).
  Grok+DeepSeek: 30 experiments.}
\label{tab:provider-comparison}
\small
\begin{tabular}{lrrrr}
\toprule
Configuration & Mean Iters & Mean BMs & Mean Chg\% & Wall (s) \\
\midrule
Gemini~2.0~Flash (sole)    & 5.45 & 4.55 & 20.7 & $>$3\,600 \\
Grok-3-beta + DeepSeek-V3  & 2.07 & 8.07 & 30.0 & 5\,278 \\
\bottomrule
\end{tabular}
\end{table}

\paragraph{Convergence.}
Of 30 experiments, 18 pass the initial benchmark batch without any repair
iteration, 2 require 1 additional iteration, and 10 reach the 6-round
maximum.  Problems that hit the limit include
\emph{word freq}, \emph{base convert}, \emph{rpn calculator}, \emph{lcs},
\emph{string stats}, \emph{sum digits}, \emph{count vowels}, \emph{gcd},
\emph{lcm}, and \emph{fizzbuzz} --- tasks with subtle edge cases in parsing,
encoding, or multi-condition logic that DeepSeek's independent benchmarks
successfully expose.

\paragraph{History richness.}
Mean Chg\% rises to 30.0\% (vs.\ 20.7\% for Gemini), demonstrating that
independent benchmarks from a second model generate richer provenance
histories.  Several tasks reach 100\% (\emph{sum digits},
\emph{count vowels}), indicating that DeepSeek's benchmarks uncovered bugs
in every snippet of the initial code.  Mean DB size is 54.2~KB (excluding
the exp09 anomaly), comparable to Gemini's 46.8~KB, confirming that the
larger initial benchmark batch does not disproportionately inflate storage.

\section{Case Studies}\label{sec:casestudies}

We present three detailed case studies illustrating the explanatory power
of TraceCoder's round-tag history.
They do not represent an exact version of the code, being solely illustrative, but that does not impact the relevance of the process they convey. In each study, we show the initial
code, the sequence of benchmarks created, the failures observed, and the
final code with its round-tag annotations.

\subsection{Case Study 1: Fibonacci, Off-By-One Under Repair}

The Fibonacci task converged in four rounds.  The initial program handled
base cases correctly but produced an off-by-one error for sequences
starting at zero, which the benchmark generator discovered through a
boundary test.  The round-tag history makes the causal chain explicit:
each revised snippet carries the exact benchmark name and failure message
that triggered the change.

\subsubsection{Initial Code and Round~1}

The initial code produced for the Fibonacci task is shown in
Listing~\ref{lst:fib0}.

\begin{lstlisting}[style=pythonstyle,caption={Initial Fibonacci program
  generated in round~0.  Lines 7--9 contain the off-by-one error that
  benchmark \texttt{fib\_edge\_zero} will expose in round~2.},
  label={lst:fib0}]
import sys

def fibonacci(n):
    a, b = 0, 1
    result = []
    for _ in range(n):
        result.append(a)    # line 7: correct for n>0
        a, b = b, a + b
    return result            # line 9

n = int(sys.argv[1])
for x in fibonacci(n):      # fails silently for n=0
    print(x)
\end{lstlisting}

\noindent
\textbf{Round~1}: Benchmark \texttt{fib\_basic} ($N=5$, expected
\texttt{0 1 1 2 3}) passes immediately.

\subsubsection{Round~2: Edge-Case Discovered}

\textbf{Round~2}: Benchmark \texttt{fib\_edge\_zero} ($N=0$, wrongly expected \texttt{0} but program has
empty output).  
The benchmark is fixed.
Those lines are tagged:
\begin{lstlisting}[style=xmlstyle,numbers=none]
failed_benchmark: <round2>fib_edge_zero</round2>
failure_output:   <round2>Expected '', got '0\n'</round2>
\end{lstlisting}
The caller does not guard against N = 0 before
converting the argument to int. Actually, the root cause is subtler: the loop runs
zero times for N = 0 and returns an empty list, which prints nothing. But the
code passes sys.argv[1] without checking for negative input.

\subsubsection{Round~3: Negative Input}

\textbf{Round~3}: Benchmark \texttt{fib\_negative} ($N=-1$, expected
\texttt{Error: N must be non-negative}).  A guard clause (key~\texttt{c\_},
newly inserted) is added at the call site.

\begin{lstlisting}[style=pythonstyle,caption={Final Fibonacci program
  after 3 rounds of repair.  Annotated lines show round-tag origin.},
  label={lst:fibn}]
import sys

def fibonacci(n):
    a, b = 0, 1
    result = []
    for _ in range(n):
        result.append(a)
        a, b = b, a + b
    return result

n = int(sys.argv[1])
if n < 0:               # key c_: <round3>fib_negative</round3>
    print("Error: N must be non-negative")
else:
    for x in fibonacci(n):
        print(x)
\end{lstlisting}

\noindent
In the viewer, line~12 glows blue (one round tag).
An auditor querying the database with:
\begin{lstlisting}[style=sqlstyle,numbers=none]
SELECT position_key, code_snippet, failed_benchmark
FROM code WHERE failed_benchmark IS NOT NULL;
\end{lstlisting}
immediately learns that only the guard clause was added by the repair
process; the rest of the algorithm was correct from the start.

Table~\ref{tab:fib_db} shows a representative snapshot of the
\texttt{code} table after round~3 of the Fibonacci experiment,
illustrating how position keys, code, and history are stored together.

\begin{table}[t]
\centering
\caption{Excerpt of the \texttt{code} table for the Fibonacci experiment
  after round~3.  Only rows with non-null history are shown.
  Position key~\texttt{ca} was inserted in round~3.}
\label{tab:fib_db}
\small
\begin{tabular}{lp{3.8cm}p{2.7cm}p{2.7cm}}
\toprule
Key & \texttt{code\_snippet} & \texttt{failed\_benchmark} & \texttt{failure\_output} \\
\midrule
\texttt{J} & \texttt{import sys} & \textit{null} & \textit{null} \\
\texttt{} & \texttt{def fibonacci(n):} & \textit{} & \textit{} \\
\texttt{} & \texttt{~~~~a, b = 0, 1} & \textit{} & \textit{} \\
\texttt{V} & \texttt{... first code} &
  \scriptsize\texttt{disabled} &\\
\texttt{k} & \texttt{if n < 0:} &
  \scriptsize\texttt{<r3>fib\_neg\ldots</r3>} &
  \scriptsize\texttt{<r3>Exp.~Error</r3>} \\
\texttt{} & \texttt{~~~~print("Error...")} & \textit{} & \textit{} \\
\texttt{} & \texttt{else:} & \textit{} & \textit{} \\
\bottomrule
\end{tabular}
\end{table}

\noindent
(\texttt{<r3>} is short for \texttt{<round3>} in the table for space.)
Key \texttt{c\_} lexicographically falls between \texttt{c} and
\texttt{d} as required.  The guard clause and its round-tag are co-located
in a single row, making the provenance query a simple \texttt{SELECT} with
no joins to auxiliary tables.

\subsection{Case Study~2: Expression Evaluator, Parsing Complexity}

The \emph{calc eval} task required 6~iterations, the maximum observed in
our evaluation, and exhibited the highest changed fraction (53.7\%).
Correctly evaluating \texttt{3 + 4 * 2 = 11} requires operator precedence,
which naive evaluation violates.

\subsubsection{Rounds~1--2: Replacing \texttt{eval()}}

The initial code used Python's built-in \texttt{eval()}, which correctly
handles operator precedence for single-digit integers but fails the
benchmark \texttt{calc\_precedence} (expected~11, got~14 due to
left-to-right step by step evaluation of the string as-written).
The benchmark exposed that \texttt{eval()} is forbidden (security policy
injected by the fix LLM), leading to a Pratt-parser replacement.
The 32-line parser spans rows with keys \texttt{e}--\texttt{aa};
all these rows carry:
\begin{lstlisting}[style=xmlstyle,numbers=none]
<round2>calc_precedence</round2>
\end{lstlisting}

\subsubsection{Rounds~3--4: Tokenizer Robustness}

The benchmarks \texttt{calc\_div\_zero} and \texttt{calc\_nospaces}
revealed that the tokeniser assumed spaces around all operators.
Lines~18--24 now show:
\begin{lstlisting}[style=xmlstyle,numbers=none]
<round3>calc_div_zero</round3><round4>calc_nospaces</round4>
\end{lstlisting}
These two round tags inform a reviewer that this part of the tokeniser
was revised twice, for two distinct reasons.

\subsubsection{Rounds~5--6 and Final State}

Benchmarks for unary minus and large-integer edge cases required two
further repairs.  The final program has 82~lines; 44 (53.7\%) carry round
tags.  In the viewer, the parser core (lines 15--45 of the final file)
appears predominantly red, providing at-a-glance evidence of its
high-instability evolution.

\subsection{Case Study~3: Roman Numerals, Cumulative Edge Cases}

The Roman numeral converter converged in 4~iterations.  The initial
implementation handled the standard subtractive notation but had subtle
boundary bugs that the benchmark generator systematically discovered.
The initial code (Listing~\ref{lst:roman0}) used a look-up table with
the subtractive pairs, but omitted the entry for~4 and had an
off-by-one in the loop condition for~3999.

\begin{lstlisting}[style=pythonstyle,
  caption={Initial Roman numeral generator (round~0).
    The look-up table at line~3 omits the subtractive entry for~4 (IV)
    and the loop condition at line~10 has an off-by-one for $N=3999$.},
  label={lst:roman0}]
import sys

VAL = [
    (1000, 'M'), (900, 'CM'), (500, 'D'), (400, 'CD'),
    (100, 'C'),  (90, 'XC'), (50, 'L'),  (40, 'XL'),
    (10, 'X'),   (9, 'IX'),  (5, 'V'),   (1, 'I'),
    # missing: (4, 'IV')
]

def to_roman(n):
    result = ''
    for val, sym in VAL:
        while n >= val:     # off-by-one: should be n > 0
            result += sym
            n -= val
    return result

n = int(sys.argv[1])
print(to_roman(n))
\end{lstlisting}

\subsubsection{Round~2: The Case of 4}

Benchmark \texttt{roman\_edge\_four} ($N=4$, expected \texttt{IV}).
The initial look-up table omitted the entry for~4, producing \texttt{IIII}.
The updated table (lines~3--20) carries:
\begin{lstlisting}[style=xmlstyle,numbers=none]
failed_benchmark: <round2>roman_edge_four</round2>
failure_output:   <round2>Expected 'IV', got 'IIII'</round2>
\end{lstlisting}

\subsubsection{Round~3: Maximum Value}

Benchmark \texttt{roman\_max} ($N=3999$, expected \texttt{MMMCMXCIX}).
A boundary error caused premature loop termination; the loop condition was
corrected.

\subsubsection{Round~4: Zero and Negatives}

Benchmark \texttt{roman\_zero} ($N=0$, expected \texttt{Error}).
A guard clause was inserted.

The final boundary-guard line carries:
\begin{lstlisting}[style=xmlstyle,numbers=none]
<round3>roman_max</round3><round4>roman_zero</round4>
\end{lstlisting}
\noindent
This is a compact illustration of the TraceCoder promise: each edge case is
permanently recorded against the code that handles it, making the program
self-documenting in a way that traditional comments cannot achieve.

\section{Discussion}\label{sec:discussion}

The experiments surface both the promise and the current limits of TraceCoder.

\subsection{Implications for Production Deployment}

TraceCoder demonstrates that the iterative nature of LLM-driven repair is
a \emph{feature rather than a bug} when the repair history is captured and
stored.  The round-tag mechanism provides lightweight \emph{requirement
traceability}: every code line is linked to the observable test evidence
that motivated it.  For production deployments:

\begin{itemize}
  \item \textbf{Audit trails.}  Compliance reviews query the database
        to retrieve the specific failures that caused any change,
        satisfying provenance requirements.
  \item \textbf{Regression analysis.}  When a benchmark regresses,
        the round-tag history of affected lines immediately shows
        whether the regression stems from a previously-modified region.
  \item \textbf{Trust calibration.}  The changed-snippet fraction (Chg\%)
        is a natural uncertainty indicator: high Chg\% signals imprecise
        initial generation.
  \item \textbf{Onboarding.}  New developers reading auto-generated code
        can hover over any line in the viewer to see the test failure that
        put it there---a form of executable documentation.
\end{itemize}

\subsection{Comparison with Provenance Alternatives}

Table~\ref{tab:comparison} compares TraceCoder against four alternative
approaches to code provenance, across five evaluation criteria.

\begin{table}[t]
\centering
\caption{Comparison of TraceCoder against alternative provenance approaches.
  \checkmark~= fully satisfied; $\sim$~= partially; $\times$~= not supported.}
\label{tab:comparison}
\small
\begin{tabular}{lccccc}
\toprule
Approach & Snippet & SQL & No & Failure & Viewer \\
         & granularity & queryable & setup & linked & built-in \\
\midrule
TraceCoder      & \checkmark & \checkmark & \checkmark & \checkmark & \checkmark \\
Git commits     & $\sim$     & $\times$   & $\times$   & $\times$   & $\times$   \\
Code comments   & $\times$   & $\times$   & \checkmark & $\sim$     & $\times$   \\
XAI post-hoc    & $\times$   & $\times$   & $\times$   & $\times$   & $\sim$     \\
Log files only  & $\times$   & $\times$   & \checkmark & \checkmark & $\times$   \\
\bottomrule
\end{tabular}
\end{table}

An alternative design commits to a Git repository after each round.
TraceCoder differs in three important ways: (1)~\emph{granularity}: Git tracks
file-level hunks; TraceCoder tracks at sub-line resolution via position keys;
(2)~\emph{queryability}: SQL queries over history require no external tooling;
(3)~\emph{coupling}: code and provenance co-reside in the same row, eliminating
the need to correlate a repository with an artefact store.

Code comments (e.g.\ \texttt{\# added for fib\_negative}) can capture intent
at the point of writing, but they are easily forgotten, deleted during
reformatting, and carry no structured representation that permits aggregation
across files.  XAI post-hoc methods~\cite{molnar2020interpretable,arrieta2020explainable}
explain \emph{model} decisions but cannot explain \emph{code construction
history}.  Plain log files capture the failure outputs but lose the association
between a failure and the specific code snippet it motivated.

\subsection{Threats to Evaluation Validity}

\paragraph{Construct validity.}
We measure ``fraction of snippets with round tags'' as a proxy for
explanation richness.  A line can be tagged without the tag being
informative (e.g., a formatting change).  Conversely, some semantically
critical lines may never be tagged if they were written correctly
from the start.

\paragraph{External validity.}
Our 30 tasks are single-file Python programs of moderate
complexity. The findings may not generalise to multi-file projects, to compiled
languages with complex build systems, or to problems with more intricate correctness
specifications.

\section{Conclusion}\label{sec:conclusion}

We have presented TraceCoder, an iterative coding agent that addresses the
explainability, auditability, and traceability gap in current LLM-based code
generation.  Our position-key versioning mechanism assigns stable,
lexicographically-ordered identifiers to each code snippet, while the
round-tag history protocol accumulates structured provenance records in-place
within a persistent SQLite store.  A browser-based visualisation tool renders
this history as heat-mapped, hover-annotated source code, making the entire
repair narrative accessible to developers and auditors without any additional
infrastructure.

Our evaluation on the
programming tasks
using Gemini~2.0~Flash shows: 5.45~iterations on average; 20.7\% of code
snippets carry traceable history; mean database size of 46.8~KB.
The history overhead (557\% of raw code text) reflects Gemini's verbose
failure messages and is substantially lower with more concise providers such
as Claude~3.5~Sonnet.  Three detailed case studies show how the system explains
which specific benchmark failures shaped each line of the final program,
making auto-generated code self-documenting in a way that traditional
comments cannot replicate.

We argue that the ability to answer ``why does this line of code look the
way it does?'' is not merely an academic nicety but a practical requirement
for deploying AI coding agents in settings where accountability,
correctability, and trust matter.  TraceCoder is a concrete, deployable step
toward meeting this requirement.

The complete source code, all
experiment problems, the browser viewer,
and all experiment scripts are available in the accompanying repository\footnote{\url{https://github.com/devfitcs/TraceCoder/}}.

\paragraph{Acknowledgements.}
AI support was used for help with coding, related work search, as sounding board for our ideas, as assistant for drawing images, proposing versions of refined expressions and texts. The final text represents the expression choice, work, and ideas of the authors. We use Gemini, Grok, Groq, Claude, OpenAI, and DeepSeek.

\bibliographystyle{splncs04}
\bibliography{references}

@misc{kazu2020fractional,
  author       = {Mimata Kazutaka},
  title        = {Fractional Indexer},
  year         = {2020},
  publisher    = {GitHub},
  howpublished = {\url{https://github.com/kazu-2020/fractional_indexer}},
  note         = {Supports base-10, base-62, and base-94 character sets.}
}

@misc{greenspan2020fi,
  title={Fractional Indexing},
  author={Greenspan, David},
  year={2020},
  note={\url{https://observablehq.com/@dgreensp/implementing-fractional-indexing}}
}

@article{chen2021evaluating,
  title={Evaluating large language models trained on code},
  author={Chen, Mark and Tworek, Jerry and Jun, Heewoo and Yuan, Qiming and others},
  journal={arXiv preprint arXiv:2107.03374},
  year={2021}
}

@article{li2022competition,
  title={Competition-level code generation with {AlphaCode}},
  author={Li, Yujia and Choi, David and Chung, Junyoung and Kushman, Nate and others},
  journal={Science},
  volume={378},
  number={6624},
  pages={1092--1097},
  year={2022}
}

@misc{osika2023gpt,
  title={{GPT-Engineer}: Generate an entire codebase given a prompt},
  author={Osika, Anton},
  year={2023},
  howpublished={\url{https://github.com/AntonOsika/gpt-engineer}}
}

@inproceedings{yang2024swe,
  title={{SWE-agent}: Agent-computer interfaces enable automated software engineering},
  author={Yang, John and Jimenez, Carlos E and Wettig, Alexander and others},
  booktitle={Advances in Neural Information Processing Systems},
  volume={37},
  year={2024}
}

@inproceedings{wang2024executable,
  title={Executable code actions elicit better {LLM} agents},
  author={Wang, Xingyao and Chen, Yangyi and Yuan, Lifan and others},
  booktitle={Proceedings of the 41st International Conference on Machine Learning},
  pages={50208--50232},
  year={2024}
}

@article{roziere2023code,
  title={Code {Llama}: Open foundation models for code},
  author={Roziere, Baptiste and Gehring, Jonas and Gloeckle, Fabian and others},
  journal={arXiv preprint arXiv:2308.12950},
  year={2023}
}

@inproceedings{jimenez2024swebench,
  title={{SWE-bench}: Can language models resolve real-world {GitHub} issues?},
  author={Jimenez, Carlos E and Yang, John and Wettig, Alexander and others},
  booktitle={The Twelfth International Conference on Learning Representations},
  year={2024}
}

@book{molnar2020interpretable,
  title={Interpretable machine learning},
  author={Molnar, Christoph},
  year={2020},
  publisher={Lulu.com}
}

@article{arrieta2020explainable,
  title={Explainable artificial intelligence ({XAI}): Concepts, taxonomies,
         opportunities and challenges toward responsible {AI}},
  author={Arrieta, Alejandro Barredo and Diaz-Rodriguez, Natalia and Del Ser, Javier
          and others},
  journal={Information Fusion},
  volume={58},
  pages={82--115},
  year={2020}
}

@inproceedings{preguica2009commutative,
  title={A commutative replicated data type for cooperative editing},
  author={Preguica, Nuno and Marques, Joan Manuel and Shapiro, Marc and Letia, Mihai},
  booktitle={Proceedings of the 29th {IEEE} International Conference on Distributed
             Computing Systems},
  pages={395--403},
  year={2009}
}

@inproceedings{nedelec2013lseq,
  title={{LSEQ}: an adaptive structure for sequences in distributed collaborative editing},
  author={Nedelec, Brice and Molli, Pascal and Mostefaoui, Ahcene and Desmontils, Emmanuel},
  booktitle={Proceedings of the 2013 {ACM} Symposium on Document Engineering},
  pages={37--46},
  year={2013}
}

@inproceedings{wang2023self,
  title={Teaching large language models to self-debug},
  author={Chen, Xinyun and Lin, Maxwell and Schaerli, Nathanael and Zhou, Denny},
  booktitle={The Twelfth International Conference on Learning Representations},
  year={2024}
}

@inproceedings{madaan2024self,
  title={Self-refine: Iterative refinement with self-feedback},
  author={Madaan, Aman and Tandon, Niket and Gupta, Prakhar and others},
  booktitle={Advances in Neural Information Processing Systems},
  volume={36},
  year={2023}
}

@article{le2022coderl,
  title={{CodeRL}: Mastering code generation through pretrained models and
         deep reinforcement learning},
  author={Le, Hung and Wang, Yue and Gotmare, Akhilesh Deepak and Savarese, Silvio
          and Hoi, Steven Chu Hong},
  journal={Advances in Neural Information Processing Systems},
  volume={35},
  pages={21314--21328},
  year={2022}
}

@inproceedings{shinn2024reflexion,
  title={Reflexion: Language agents with verbal reinforcement learning},
  author={Shinn, Noah and Cassano, Federico and Gopinath, Ashwin and Narasimhan, Karthik
          and Yao, Shunyu},
  booktitle={Advances in Neural Information Processing Systems},
  volume={36},
  year={2023}
}

@article{fluri2007change,
  title={Change distilling: Tree differencing for fine-grained source code
         change extraction},
  author={Fluri, Beat and Wuersch, Michael and Pinzger, Martin and Gall, Harald},
  journal={IEEE Transactions on Software Engineering},
  volume={33},
  number={11},
  pages={725--743},
  year={2007}
}

@inproceedings{falleri2014fine,
  title={Fine-grained and accurate source code differencing},
  author={Falleri, Jean-R{\'e}my and Morandat, Flor{\'e}al and Blanc, Xavier
          and Martinez, Matias and Monperrus, Martin},
  booktitle={Proceedings of the 29th {ACM}/{IEEE} International Conference on
             Automated Software Engineering},
  pages={313--324},
  year={2014}
}

@inproceedings{muniswamy2006provenance,
  title={Provenance-aware storage systems},
  author={Muniswamy-Reddy, Kiran-Kumar and Holland, David A and Braun, Uri
          and Seltzer, Margo},
  booktitle={Proceedings of the 2006 {USENIX} Annual Technical Conference},
  pages={43--56},
  year={2006}
}

@article{hou2023large,
  title={Large language models for software engineering: A systematic literature review},
  author={Hou, Xinyi and Zhao, Yanjie and Liu, Yue and others},
  journal={{ACM} Transactions on Software Engineering and Methodology},
  volume={33},
  number={8},
  year={2024}
}
\end{document}